\documentclass{article}

\PassOptionsToPackage{numbers,sort}{natbib}
\usepackage{fullpage}

\usepackage{natbib}
\usepackage[utf8]{inputenc} % allow utf-8 input
\usepackage[T1]{fontenc}    % use 8-bit T1 fonts
\usepackage[colorlinks,allcolors=blue]{hyperref}       % hyperlinks
\usepackage{url}            % simple URL typesetting
\usepackage{booktabs}       % professional-quality tables
\usepackage{amsfonts}       % blackboard math symbols
\usepackage{nicefrac}       % compact symbols for 1/2, etc.
\usepackage{microtype}      % microtypography
\usepackage{xcolor}

\usepackage{amsmath,amsthm}

\let\citealp\citep

\title{Stochastic Optimization with Laggard Data Pipelines}

\author{
  Naman Agarwal$^1$ \qquad Rohan Anil$^2$ \qquad Tomer Koren$^{23}$ \\
  Kunal Talwar$^4$\footnote{Work performed while at Google Brain.} \qquad Cyril Zhang$^5$\footnote{Work performed while at Google AI Princeton and Princeton University.}\\
  \vspace{-2mm} \\
  {\smaller$^1$ Google AI Princeton} \qquad
  {\smaller$^2$ Google Research} \qquad
  {\smaller$^3$ Tel Aviv University} \\
  {\smaller$^4$ Apple} \qquad
  {\smaller$^5$ Microsoft Research} \\
  \vspace{-2mm} \\
  {\smaller \texttt{\{rohananil, namanagarwal\}@google.com, tkoren@tauex.tau.ac.il},} \\
  {\smaller \texttt{ktalwar@apple.com, cyrilzhang@microsoft.com} }
}

\usepackage{amsmath}
\usepackage{mathtools}
\usepackage{amsthm} % better to load after ams math
\usepackage{latexsym}
\usepackage{relsize}
\usepackage{xspace}

\usepackage[capitalise]{cleveref}
\usepackage{xcolor}
\usepackage{dsfont}

\newcommand{\defeq}{\stackrel{\text{def}}{=}}

\newcommand{\D}{\mathcal{D}}

\newcommand{\A}{\mathcal{A}}
\newcommand{\W}{\mathcal{W}}

\newcommand{\y}{\ensuremath{\mathbf y}}

\def\y{\mathbf{y}}

\newcommand{\ignore}[1]{}

\theoremstyle{plain}
\newtheorem{theorem}{Theorem}
\newtheorem{lemma}[theorem]{Lemma}

\newtheorem*{theorem*}{Theorem}
\newtheorem*{lemma*}{Lemma}
\newtheorem*{corollary*}{Corollary}
\newtheorem*{proposition*}{Proposition}
\newtheorem*{claim*}{Claim}
\newtheorem*{fact*}{Fact}
\newtheorem*{observation*}{Observation}

\theoremstyle{definition}
\newtheorem{definition}[theorem]{Definition}

\newtheorem*{definition*}{Definition}
\newtheorem*{remark*}{Remark}
\newtheorem*{example*}{Example}

 \theoremstyle{plain}
\newtheorem*{theoremaux}{\theoremauxref}
\gdef\theoremauxref{1}

\DeclareMathAlphabet{\mathbfsf}{\encodingdefault}{\sfdefault}{bx}{n}

\DeclareMathOperator*{\argmin}{arg\,min}

\newcommand{\abs}[1]{|#1|}
\newcommand{\norm}[1]{\|#1\|}

\DeclareMathOperator*{\E}{\mathbb{E}}

\newcommand{\reals}{\mathbb{R}}

\renewcommand{\leq}{~\le~}
\renewcommand{\geq}{~\ge~}

\let\oldtfrac\tfrac
\renewcommand{\tfrac}[2]{\smash{\oldtfrac{#1}{#2}}}

\newcommand{\braces}[1]{\left\{#1\right\}}
\newcommand{\pa}[1]{\left(#1\right)}
\newcommand{\ang}[1]{\left<#1\right>}
\newcommand{\bra}[1]{\left[#1\right]}
\renewcommand{\abs}[1]{\left|#1\right|}

\DeclareMathOperator*{\Minimize}{\mathrm{minimize}}

\def\eqref#1{equation~\ref{#1}}
\def\1{\bm{1}}

\DeclareMathAlphabet{\mathsfit}{\encodingdefault}{\sfdefault}{m}{sl}
\SetMathAlphabet{\mathsfit}{bold}{\encodingdefault}{\sfdefault}{bx}{n}

\newcommand{\R}{\mathbb{R}}

\usepackage{algorithm}
\usepackage[noend]{algpseudocode}

\usepackage[extdef=true]{delimset}
\usepackage{enumitem}
\usepackage{multirow}
\usepackage{hyperref}
\usepackage{array}

\crefformat{appendix}{the supplementary material} 
\newcommand{\fBar}{\bar{f}}

\newcommand{\out}{\mathrm{out}}
\newcommand{\pivot}{\mathrm{pivot}}
\newcommand{\init}{\mathrm{init}}
\newcommand{\batch}{\pmb{\xi}}

\newcolumntype{P}[1]{>{\centering\arraybackslash}p{#1}}
\newcolumntype{M}[1]{>{\centering\arraybackslash}m{#1}}

\begin{document}
\maketitle

\begin{abstract}
State-of-the-art optimization is steadily shifting towards massively parallel pipelines with extremely large batch sizes. As a consequence, CPU-bound preprocessing and disk/memory/network operations have emerged as new performance bottlenecks, as opposed to hardware-accelerated gradient computations. In this regime, a recently proposed approach is data echoing (Choi et al., 2019), which takes
repeated gradient steps on the same batch while waiting for fresh data to arrive from upstream. We provide the first convergence analyses of ``data-echoed'' extensions of common optimization methods, showing that they exhibit provable improvements over their synchronous counterparts.
Specifically, we show that in convex optimization with stochastic minibatches, data echoing affords speedups on the curvature-dominated part of the convergence rate, while maintaining the optimal statistical rate.
\end{abstract}

\section{Introduction}

Recent empirical successes in large-scale machine learning have been powered by
massive data parallelism and hardware acceleration, with batch sizes trending beyond 10K+ images \cite{you2017large} or 1M+ tokens \cite{brown2020language}.
Numerous interdisciplinary sources \cite{ben2019demystifying,chien2018characterizing, jouppi2017datacenter,mayer2020scalable} indicate that the performance bottlenecks of contemporary deep learning pipelines can lie in many places other than gradient computation. In other words, since the initial breakthroughs in hardware-accelerated deep learning \cite{raina2009large,ciregan2012multi,krizhevsky2012imagenet}, GPUs (and TPUs, FPGAs, etc.) have become too fast for upstream data loaders and preprocessors to keep up with.

\citet{choi2019faster} propose \emph{data echoing}, a simple and versatile way to improve training performance in this regime. Each stage of the data pipeline runs asynchronously, oblivious to
whether its input has been refreshed upstream. In particular, the optimization algorithm may choose to take additional gradient steps before a minibatch is
refreshed, rather than spend idle time waiting for more data. The authors
present a large-scale proof-of-concept empirical study, and find that data
echoing affords a $3.25\times$ speedup in a network-bound ImageNet setting.

Some natural curiosities arise from this practice: \emph{When might this
overfit? How carefully should one adjust the step size of an echoed gradient?
Does acceleration work?} A theoretical understanding of convergence guarantees
for these data-echoed optimization algorithms is missing.

In this paper, we settle the issues of convergence and generalization for echoed gradient methods in convex optimization. We show that these methods can match the optimization performance of their non-stochastic counterparts, while achieving optimal statistical rates. As state-of-the-art batch sizes continue to grow, along with the distributed systems that enable them, we hope that this will provide a first theoretical grounding towards understanding the algorithmic and statistical challenges in these hardware-motivated optimization settings.

\begin{figure}[t]
  \centering
  \includegraphics[width=0.5\linewidth]{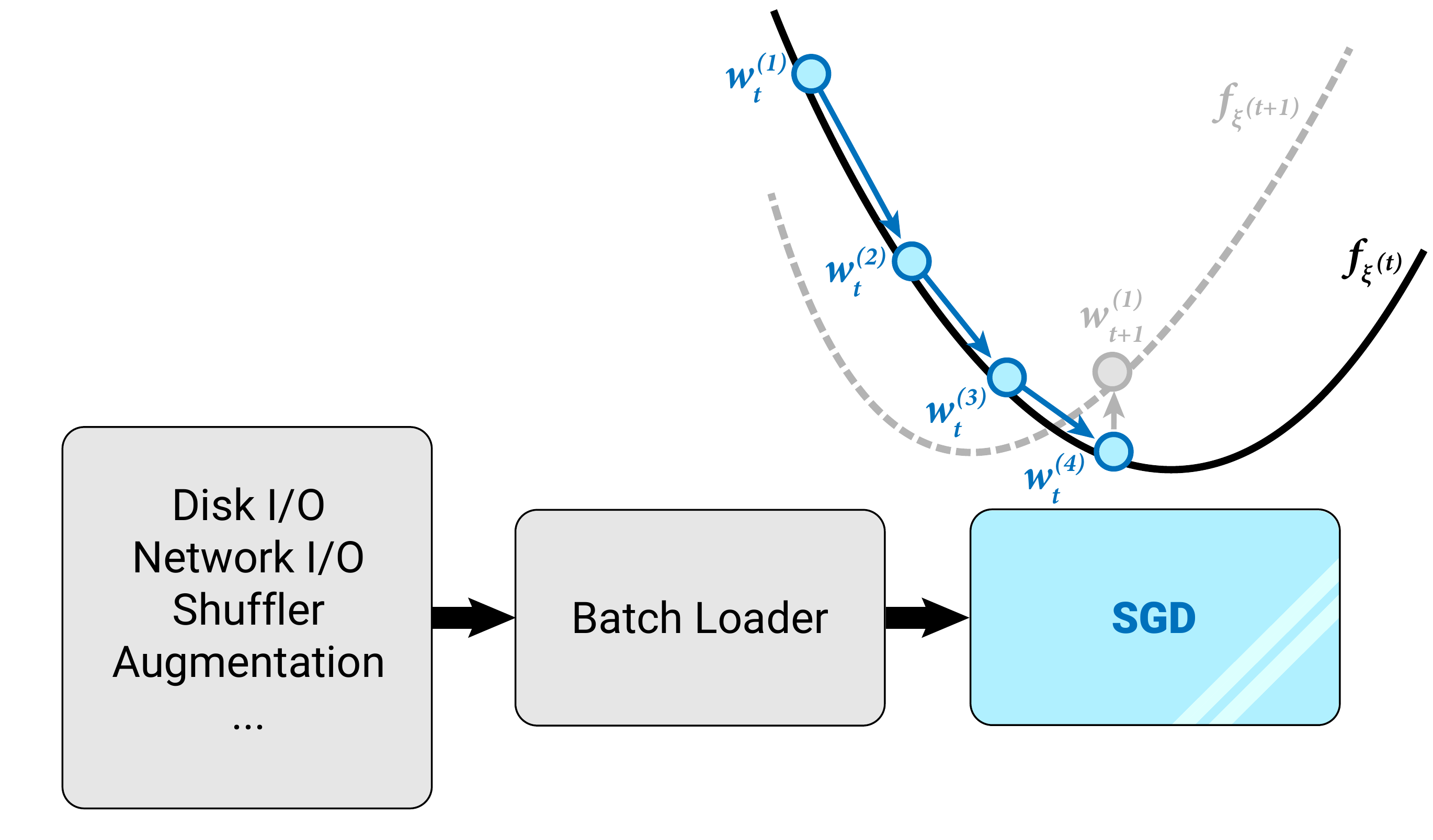}
  \caption{Schematic of data echoing, inspired by \citet{choi2019faster}. If the upstream data pipeline is $K=4$ times slower than SGD, then SGD can potentially take that many steps on the same batch before the next one arrives.}
  \label{fig:my_label}
\end{figure}

\subsection{Technical contributions}

Our model of data echoing is parameterized by the batch size $B$, the number of
fresh i.i.d.~batches $T$, and the \emph{echoing factor} $K$, which is the number
of gradient steps an algorithm can take on the (convex) loss on each batch. This
reflects the hardware-determined setting where the data loader is at least $K$ times slower than the optimizer.

\paragraph{Convergence in all data echoing regimes.} 

We first show that echoed SGD, with the correctly tuned step size, achieves a factor-$K$ speedup on the curvature term of the standard convergence rate, while keeping the optimal statistical term. Next, we develop an echoed method
that is oblivious to the echoing factor $K$, getting the same rates for echoed
SGD with an appropriately chosen proximal regularizer. Finally,
we show that Nesterov's accelerated gradient descent, when echoed, achieves the optimal rates on quadratic losses. As a side contribution, we fix a small error in a technical lemma
in \citep{chen2018stability}, used in establishing the stability of AGD on quadratics. For general convex losses, we arrive at the same open question as these authors.

\begin{table}
\centering
\begin{tabular}[c]{| M{1.2cm} | M{3.2cm} | M{3.5cm} |}
\hline
& $T=1$ & $T$ general \\\hline 
$K=1$ & \multicolumn{2}{c|}{SGD}\\\hline
$K$ small & Compute-bound ERM & Data echoing (Thm.~\ref{thm:degd}) \\\hline
$K$ large & Data-bound ERM & Approx-Prox \cite{wang2017memory} \\\hline
$K \rightarrow \infty$ & Statistical ERM & Minibatch-Prox \cite{wang2017memory}\\
\hline
\end{tabular}
\vspace{3ex}
\caption{Regimes of echoing factor $K$ and number of batches $T$ which our analyses interpolates.}
\label{table:regimes}
\end{table}

\paragraph{Full interpolation between known regimes.}

To set up notation, suppose that we go over $T$ batches of data, and perform $K$
echoed gradient steps for each batch. In the special case of $T=1$ fresh
batches, the problem becomes empirical risk minimization with a limited
computational budget of $K$ gradient steps. When $K$ is small, the error is
dominated by a \emph{curvature} term, while for large enough $K$ this falls
below the \emph{statistical} term.

Motivated by the communication-limited setting, \citet{wang2017memory} focus
on the case where $T$ is general and $K\rightarrow \infty$, analyzing the
convergence of \emph{exact} optimization of the prox-regularized minibatch loss.
They develop a mild ``approx-prox'' guarantee when $K$ is large enough to enable
an \emph{exact+perturbation} analysis. Our analysis generalizes and strengthens these results, handling all values of $K$; \cref{table:regimes} summarizes
this discussion.
When $B \rightarrow \infty$, the statistical problem disappears, and we recover the classical setting of full gradient descent with $KT$ oracle calls \citep{boyd2004convex,nesterov2013introductory}.

\begin{table*}
\centering
    \begin{tabular}[c]{|M{2.8cm}|M{4.5cm}|M{4.5cm}|}
         \hline
         {\bf Algorithm} & \begin{tabular}{c}
              {\bf Standard}
         \end{tabular} & \begin{tabular}{c}
              {\bf Data-echoed}
         \end{tabular} \\ 
         \hline
         \hline
         SGD & \begin{tabular}{c}
              $\displaystyle
              O\left( \frac{\beta D^2}{T} + \frac{\rho D}{\sqrt{BT}}\right)$
              \\[2ex]
              (classical; see \citet{lan2012optimal})
         \end{tabular} & \begin{tabular}{c}
              $\displaystyle
              O\left( \frac{\beta D^2}{KT} + \frac{\rho D}{\sqrt{BT}}\right)$
              \\[2ex]
              (\cref{thm:degd})
         \end{tabular} 
         \\[2ex]
         \hline
         \begin{tabular}{c}
              Minibatch-Prox
         \end{tabular} & \begin{tabular}{c}
             $\displaystyle
             O\left( e^{-K/\kappa} + \frac{\rho D}{\sqrt{BT}}\right)$ 
             \\[2ex]
             (\citet{wang2017memory}; $K$ large)
         \end{tabular} & \begin{tabular}{c}
               $\displaystyle
               O\left( \frac{\beta D^2}{KT} + \frac{\rho D}{\sqrt{BT}}\right)$
               \\[2ex]
               (\cref{thm:depgd})
          \end{tabular}
          \\
         \hline
         \begin{tabular}{c}
              Stochastic AGD
         \end{tabular} & \begin{tabular}{c}
              $\displaystyle
              O\left( \frac{\beta D^2}{T^2} + \frac{\rho D}{\sqrt{BT}}\right)$
              \\[2ex]
              (\citet{lan2012optimal})
         \end{tabular} & \begin{tabular}{c}
              $\displaystyle
              O\left( \frac{\beta D^2}{K^2T^2} + \frac{\rho D}{\sqrt{BT}}\right)$
              \\[2ex]
              (\cref{thm:deagd}; quadratics)
          \end{tabular} \\
         \hline
    \end{tabular}
    \caption{Single-step and \emph{data-echoed} convergence rates of stochastic
         optimization algorithms studied in this paper. Notice that the
         optimization terms depend analogously on the total number of steps
         $KT$, and the statistical terms have optimal dependence on the total
         number of i.i.d.~samples $BT$.}
\label{table:rates}
\end{table*}

\paragraph{Stability-based analysis.}

We provide a modular proof framework for data echoing convergence bounds, based on uniform stability \cite{bousquet2002stability} and a potential-based notion of regret, which isolates the ``bias'' (curvature) and ``variance'' (generalization) components of the problem. This recipe (Theorem~\ref{thm:main}) can be used to sharpen bounds in more restricted settings, or analyze future data-echoed algorithms.

\subsection{Motivation and context}

It is well-known in the practice of GPU training that model parameter updates are not necessarily the performance bottleneck; this is why SSD storage is critical for pipelines on the scale of ImageNet \cite{imagenet}. For quantitative studies of I/O performance in deep learning, see
\cite{chien2018characterizing,ying2018image}. Many empirical advances have stemmed from innovations in data augmentation \cite{cubuk2019autoaugment,hoffer2019augment,shorten2019survey}. Unlike neural
network training and inference, these data transformations can be highly sequential and/or heterogeneous, and must be done on CPU. Unlocking GPU parallelism for CPU-bound computations is often a significant engineering effort \cite{dalton2019gpu,guirao2019fast,khadatare2020leveraging,liang2018gpu}.

Extremely large batch sizes have become the norm in training state-of-the-art
models
\citep{bapna2019massively,brown2020language,devlin2018bert,shazeer2017outrageously,you2017large}.
An overwhelming theme has been that \emph{constant factors matter}; for example, memory-bound optimizers \citep{shazeer2018adafactor,anil2019memory,chen2019extreme} care about factors of 2-3.
When selecting hyperparameters in large-batch training setups, it is common to balance the curvature- and noise-dominated terms \cite{mccandlish2018empirical,kaplan2020scaling,shallue2018measuring,smith2017don}. This underscores the need to better understand the fine-grained dependences on $B$, $T$, and $K$, especially the resources at stake are on the scale of GPU-years.

The idea of repeated steps on a batch/workers has also been investigated in the context of federated learning \cite{kairouz2019advances}, where a related concept is referred to as \emph{local SGD} or \emph{federated averaging}. There are two key distinctions: federated learning considers multiple copies of local SGD running on different workers, which synchronize intermittently through averaging; under the most simplified assumptions, each individual gradient step within a worker is taken on a fresh batch. While the improvements obtained in this recent and concurrent line of work (see \cite{woodworth2020local} and references therein) bear resemblance to our bounds, we do not see a direct reduction in either direction. Indeed, due to the distinctions mentioned, getting similar improvements to the curvature term in federated learning is not possible beyond quadratics, as shown by \cite{woodworth2020local}. Obtaining optimal rates for convex functions in the federated learning setting remains an interesting open problem.

\subsection{The bias-variance problem in data echoing}
As mentioned earlier, data echoing presents a natural tradeoff between the
optimization gains from repeating gradient steps vs. the potential loss of
generalization due to overfitting to stale batches. To understand this
in detail, let us revisit the standard convergence guarantee for SGD on smooth
functions:
\[ 
  \E[F(w_\out)] - F(w^*) \leq O\left(\frac{1}{T} + \frac{1}{\sqrt{BT}}\right).
\] 
We interpret the first term as a \emph{bias} (\emph{curvature}) term, which
diminishes at a faster rate due to smoothness. The second term is the
\emph{variance} (\emph{statistical}) term, which arises due to the
stochasticity in the data, and thus naturally scales as the inverse square root
of the batch size. Viewing $B$ as fixed, the variance term is intrinsic to the
data; therefore, we cannot expect data-echoing (or any algorithm) to give us
improvements on that term for free. In fact, it is possible to make this term
degrade, by overfitting on a batch. On the other hand, we can expect the bias
term, which is governed by progress on the curvature of the underlying population loss, to decrease
as we are given more echoing steps $K$. In light of this, the best analogous convergence
rate one should hope to achieve in the data-echoing setting is
\[
  O\left(\frac{1}{KT} + \frac{1}{\sqrt{BT}}\right).
\] 
Our results establish exactly this rate for the data-echoed version of gradient
descent. The data-echoed version of accelerated gradient descent is also shown
to possess similar gains but with a faster rate of $K^2T^2$. The challenge is to
prevent overfitting; obtaining such rates requires careful control (depending on
$K$) of step sizes. Later, we alleviate this need via data-echoed proximal GD,
whose parameters are independent of $K$.  

\subsection{Overview of techniques}
All of our theorems follow the same analysis structure. In particular, we
formalize a notion of \emph{potential-bounded regret} (\cref{def:pbdregret}),
which connects an algorithm's function-value progress on a minibatch to a
decrement on a certain potential function with respect to an arbitrary point.
This potential function depends on the algorithm in question, but the key
property is that it telescopes when summed over batches; this provides a fast
rate on the bias term with respect to $T$.

The second piece of the analysis connects function-value decrease on a batch to
the population objective via the notion of \emph{uniform stability}
(\cref{def:stab}). Note that the potential decrease scales inversely with $K$,
whereas the stability constant increases with $K$ (unless a proximal regularizer is added).
The key to maintaining the optimal statistical rate is to balance these terms via the choice of an appropriate step size. This type of algorithmic stability analysis has appeared various times in the literature \cite{bousquet2002stability,hardt2015train,chen2018stability}; we show here that it affords a way to analyze echoed gradient methods.

\section{Preliminaries}
\subsection{Problem definition}
\noindent Given a convex set $\mathcal{W} \subseteq \reals^n$ and a  domain $\Xi$ with a distribution $\D$, we consider the following stochastic convex optimization problem:
\begin{equation}
    \Minimize_{w \in \mathcal{W}} \quad F(w) \defeq \E_{\xi \sim \mathcal{D}}[ f(w, \xi) ]
    .
\end{equation}
Here $f : \reals^n \times \Xi \rightarrow \reals$ is such that for any $\xi$, $f(\cdot, \xi)$ is convex, differentiable, $\rho$-Lipschitz, and $\beta$-smooth; i.e., for all $w, w' \in \mathcal{W},$
\[
  f(w) - f(w') \leq \ang{\nabla f(w'), w - w'} + \frac{\beta}{2}\|w - w'\|^2.
\]
When the minimizer exists, we define $w^* = \argmin_{w\in\W} F(w)$. However, our results pertaining to optimality gaps $F(w) - F(w^*)$ hold for arbitrary $w^*$, encompassing the case when this minimizer does not exist. We further assume that we have access to an initial point $w_0$ with a bounded distance $D$ from the comparator; i.e.,
$\|w_0 - w^*\| \leq D.$
\paragraph{Minibatch optimization.} We will work in the stochastic minibatch oracle model: at each time step $t$, we receive a new batch (of size $B$) examples 
$\batch^{(t)} = \{\xi^{(t,i)}\}_{i=1}^{B}$
sampled i.i.d.~from the distribution $\D$. %For ease of notation let $f^{(t,i)}(w)$ denote $f(w, \xi^{(t,i)})$. 
For any batch of examples $\batch = \{\xi^{(i)}\}$, we define the empirical objective on the batch as
$$\fBar_{\batch}(w) \defeq \frac{1}{|\batch|} \sum_{i=1}^B f(w, \xi^{(i)}). $$
Throughout this paper, we will use \textbf{boldface} $\batch$ to denote a batch of $B$ examples, and unbolded $\xi$ to represent a single example in $\Xi$.
\paragraph{Optimization algorithms.}
We formalize a generic notion of optimization algorithms. Since these algorithms are called repeatedly by the data-echoing procedure, we will augment the output space of optimization algorithms with a notion of \textit{state}, which it internally maintains and passes to the next run of the same algorithm. Formally, an optimization algorithm is an iterative procedure which takes four arguments: an initial point $w_{\init} \in \W$, an initial state $s_{\init}$, the current batch $\batch$ which determines the current objective $\fBar_{\batch}$, and the number of steps $k$. The algorithm outputs a point $w_{\out} \in \W$ and an output state $s_{\out}$. In short, an algorithm $\A$ implements
\[( w_{\out}, s_{\out} ) \leftarrow \A(w_{\init}, s_{\init}, \batch, k).\] We will suppress the notation of one or more of the arguments to $\A$ when they will be
clear from the context, and write $f(\A(\cdot))$ as a shorthand for $f(w_{\out})$, ignoring the auxiliary state $s_{\out}$. Note that $w_{\out}$ and $s_{\out}$ are random variables, determined by the stochastic minibatch $\batch$.

\subsection{Algorithmic stability}

\begin{definition}[Uniform stability]
\label{def:stab}
A deterministic%
\footnote{A similar definition exists for randomized algorithms \cite{bousquet2002stability}. In this work, we focus on deterministic algorithms.}
algorithm $\A$ is considered to be $\epsilon$-uniformly stable with respect to loss function $f : \mathcal{W} \times \Xi \rightarrow \R$ if, for two
batches of data $\batch, \batch'$ differing in exactly one example, we have that
\[\sup_{\xi \in \Xi} \abs{ \; f(\A(\batch),\xi) -f(\A(\batch'), \xi)\;} \leq \epsilon.\]
\end{definition}
The following is a well-known result connecting stability to
generalization~\cite{bousquet2002stability}. Here, we state a version taken
from~\cite{hardt2015train}:
\begin{theorem}
If an algorithm $\A$ is $\epsilon$-uniformly stable, then it holds that
\[\left\lvert \, \E_{\batch \sim \D^B}\left[\;\fBar_{\batch}\pa{\A(\batch)} - F\pa{\A(\batch)}\;\right] \right\rvert \leq \epsilon. \]
\end{theorem}

\section{The data echoing meta-algorithm}

Given an minibatch optimization algorithm $\A$, its data-echoed extension is defined by Algorithm~\ref{alg:dataecho}.

\begin{algorithm}[h!] 
\caption{Data echoing meta-algorithm}
\label{alg:dataecho}
\begin{algorithmic}[1]
\State \textbf{Input: } Optimizer $\A$; initializer $w_{\init} := w_0$; initial state $s_{\init} := s_0$; number of inner steps $K$
\For{$t = 0, \ldots, T-1$}
  \State Receive a batch of examples $\batch^{(t)} = \{\xi^{(t,i)}\}_{i = 1}^{B}$.
%   \State Define the batch objective
%   \[\fBart{t}(w) \defeq \frac{1}{B} \sum_{i=1}^B f(w, \xi^{(t,i)})\]
  \State Execute $\A$ on $\batch^{(t)}$ starting at $w_t$ for $K$ steps:
  \quad
  $(w_{t+1},s_{t+1}) \leftarrow \A(w_{t}, s_t, \batch^{(t)}, K).$
 \EndFor
\State \textbf{Output: } Average iterate $w_{\out} := \frac{1}{T} \sum_{t=0}^{T-1} w_{t}$
\end{algorithmic}
\end{algorithm}

% \tomer{can't we include the averaged iterate as state and avoid the explicit averaging? then we can unify the two meta algos}

\subsection{Data-echoed algorithms}

Using the framework of Algorithm~\ref{alg:dataecho}, we introduce the data-echoed versions of three ubiquitous optimization algorithms. In \cite{choi2019faster}, several types of data echoing are defined; we focus on what the authors call \emph{batch echoing}.

\paragraph{Data-echoed gradient descent.}
We first formalize gradient descent in our optimization framework. The gradient descent procedure only contains the \textit{fixed} learning rate as the state:
\[s_{\init} = s_{\out} := \{\eta\}.\]
The iterations defining the inner algorithm $\A$ are straightforward:
\[w_0 = w_{\init},\;\; \{w_{j+1} = w_j - \eta \nabla \fBar_{\batch}(w_j)\}_{j=0}^{K-1},\;\; w_{\out} = w_K.\]
When Algorithm~\ref{alg:dataecho} is instantiated with this choice of $\A$, we call the overall procedure \emph{data-echoed gradient descent}.

\paragraph{Data-echoed proximal gradient descent.}
The state of the proximally-regularized gradient descent procedure contains three variables: the fixed learning rate $\eta$, the prox parameter $\gamma$, and $w_{\pivot}$, the center of the prox term:  
\[s_{\init} := \{\eta, \gamma, w_{\pivot}\}.\]
We now define the proximal function 
\[\fBar_{\mathrm{prox}}(w) = \fBar_{\batch}(w) + \frac{\gamma}{2}\|w-w_{\pivot}\|^2.\]
The iterations proceed in same way as gradient descent, but on $\fBar_{\mathrm{prox}}$:
\[w_0 = w_{\init},\;\; \{w_{j+1} = w_j - \eta \nabla \fBar_{\mathrm{prox}}(w_j)\}_{j=0}^{K-1},\;\; w_{\out} = w_K.\]
The output returned is
$s_{\out} = \{ \eta, \gamma, \frac{1}{K}\sum_{j=0}^{K-1} w_j \}$. This particular choice of returning the average iterate as the next $w_{\pivot}$ simplifies our analysis. With this choice of $\A$, this overall procedure will be called \emph{data-echoed proximal gradient descent}.
\paragraph{Data-echoed accelerated gradient descent.}
The state space for accelerated gradient consists of a step size $\eta$, an initial momentum vector $d$, and a momentum scale factor $\lambda$; thus
$s_{\init} = \{\eta, d, \lambda\}.$
Define the following scalar sequences with $\lambda_0 = \lambda$: 
\[
    \lambda_{j+1}^2 - \lambda_{j+1} = \lambda_{j}^2 ,
    \quad\quad 
    \gamma_{j+1} = \frac{\lambda_j - 1}{\lambda_{j+1}} .
\]
The updates now follow the progression as in Nesterov's acceleration \cite{nesterov1983method}: 
\begin{align*}
    w_0 = w_{\init},\;\;d_0 = d, \quad 
    w_{j+1} = (w_j + d_j) - \eta \nabla \fBar_{\batch}(w_j + d_j),
\quad
    d_{j+1} = \gamma_{j+1} (w_{j+1} - w_j)
    .
\end{align*}
Finally, the outputs are given by $s_{\out} = \{\eta, d_K, \lambda_K\}, w_{\out} = w_{K}$.

With this choice of $\A$, we refer to the overall procedure as \emph{data-echoed accelerated gradient descent}. 

\section{Convergence analyses of echoed methods}

We will analyze the data-echoing algorithms by separating their optimization properties from their stability properties. For the latter, we use the standard notion of uniform stability, as defined earlier. For the optimization part, we use a notion of potential-bounded regret, which we define next.

\begin{definition}[Potential-bounded regret]
\label{def:pbdregret}
We say that an algorithm $\A$ has \emph{potential-bounded regret} with potential function $V_{\A}$ if given a $\beta$ smooth convex function $f$ on a domain $\W$ and a starting point $w_{\init}$, $\A$ produces a point $w_{\out}$ such that for all $w^* \in \W$, it holds that
\[f(w_{\out}) - f(w^*) \leq V_{\A}(w_{\init}, s_{\init}, w^*) - V_{\A}(w_{\out}, s_{\out}, w^*).\]
\end{definition}

This inequality is a fundamental lemma in the standard analysis of mirror descent (see \cite{ben2001lectures}, or Section~B.2 from \cite{allen2014linear}), but we extend it to \emph{nested stateful algorithms} instead of a single step. For the echoed algorithms we analyze in this work, squared Euclidean norms will be suitable potentials.

We state and prove our main generic theorem below:
\begin{theorem}
\label{thm:main}
Let $\A$ be an $\epsilon$-uniformly stable algorithm. Furthermore, suppose $\A$ has the potential-bounded regret property with respect to $V_{\A}$. Then, for any $w^* \in \W$, Algorithm~\ref{alg:dataecho} with inner algorithm $\A$ satisfies
\begin{align*}
  \E[F(w_{\out})] - F(w^*)  \leq \frac{V_{\A}(w_0, s_0, w^*) - \E[V_{\A}(w_{T}, s_{T}, w^*)]}{T} + \epsilon.
\end{align*}
\end{theorem}

\begin{proof}
From the potential-bounded regret property of the algorithm $\A$, we get that
\begin{align*}
  \fBar_{\batch^{(t)}}(w_{t+1}) - \fBar_{\batch^{(t)}}(w^*) \leq V_{\A}(w_{t}, s_{t}, w^*) - V_{\A}(w_{t+1}, s_{t+1}, w^*).
\end{align*}

Let $\E_t[\cdot]$ denote the expectation conditioned on all randomness in the minibatches up to (and including) time $t$. We now get from the uniform stability of $\A$ that
\begin{align*}
  \E[F(w_{t+1})] 
  = \E_{t-1} \E_{\batch^{(t)}}[F(w_{t+1})]
  \leq  \E_{t-1}\bra{\E_{\batch^{(t)}}[\fBar_{\batch^{(t)}}(w_{t+1})] + \epsilon}
  . 
\end{align*}
Thus we have
\begin{align*}
\E[F(w_{t+1})] - F(w^*) 
&\leq \E_{t-1} \E_{\batch^{(t)}}[\fBar_{\batch^{(t)}}(w_{t+1}) - \fBar_{\batch^{(t)}}(w^*)] + \epsilon \\
&\leq \E_{t-1} \E_{\batch^{(t)}}[V_{\A}(w_{t}, s_t, w^*)-V_{\A}(w_{t+1}, s_{t+1}, w^*)] + \epsilon \\
&\leq \E[V_{\A}(w_{t}, s_{t}, w^*)]-\E[V_{\A}(w_{t+1}, s_{t+1}, w^*)] + \epsilon.
\end{align*}
Summing the above over time and using the convexity of $F$ gives us that
\begin{align*}
  \E[F(w_{\out})] - F(w^*) &\leq \sum_{t=0}^{T-1} \frac{\E[F(w_{t+1})] - F(w^*) + \epsilon}{T} \\ &\leq  \frac{V_{\A}(w_0, s_0, w^*) - \E[V_{\A}(w_{T}, s_{T}, w^*)]}{T} + \epsilon
  .
  &&\qedhere
\end{align*}
\end{proof}
% \begin{proof}[Proof of Theorem \ref{thm:main}]
% From the potential bounded regret property of the algorithm $\A$ we get that
% \begin{align*}
%   \fBar_{\batch^{(t)}}(w_{t+1}) - \fBar_{\batch^{(t)}}(w^*) \leq V_{\A}(w_{t}, s_{t}, w^*) - V_{\A}(w_{t+1}, s_{t+1}, w^*).
% \end{align*}

% Let $\E_t[\cdot]$ denote expectation conditioned on all randomness up to (and
% including) time $t$. We now get from the uniform stability of $\A$ that
% \begin{align*}
%   \E[F(w_{t+1})] 
%   = \E_{t-1}[\E_{\batch^{(t)}}[F(w_{t+1})]] 
%   \leq  \E_{t-1}[\E_{\batch^{(t)}}[\fBar_{\batch^{(t)}}(w_{t+1})] + \epsilon]
%   . 
% \end{align*}
% Thus we have
% \begin{align*}
% \E[F(w_{t+1})] - F(w^*) 
% &\leq \E_{t-1}[\E_{\batch^{(t)}}[\fBar_{\batch^{(t)}}(w_{t+1}) - \fBar_{\batch^{(t)}}(w^*)]] + \epsilon \\
% &\leq \E_{t-1}[\E_{\batch^{(t)}}[V_{\A}(w_{t}, s_t, w^*)-V_{\A}(w_{t+1}, s_{t+1}, w^*)]] + \epsilon \\
% &\leq \E[V_{\A}(w_{t}, s_{t}, w^*)]-\E[V_{\A}(w_{t+1}, s_{t+1}, w^*)] + \epsilon.
% \end{align*}
% Summing the above over time and using the convexity of $F$ gives us that
% \begin{align*}
%   \E[F(w_{\out})] \leq \sum_{t=0}^{T-1} \frac{\E[F(w_{t+1})] - F(w^*) + \epsilon}{T} \leq  \frac{V_{\A}(w_0, s_0, w^*) - \E[V_{\A}(w_{T}, s_{T}, w^*)]}{T} + \epsilon
%   .
%   &&\qedhere
% \end{align*}
% \end{proof}

% \subsection{Applications of \cref{thm:main}}

In the rest of the section, we present various applications of our main data echoing theorem. In each case, we will consider a standard algorithm, derive its stability and potential bounded regret properties, then use \cref{thm:main} to derive the convergence rate for its echoed version. All regret proofs can be found in Appendix~\ref{sec:regret-proofs}, and stability proofs in Appendix~\ref{sec:stability-proofs}; the corresponding convergence rates for the echoed algorithms are proven in Appendix~\ref{app:proofs-theorems}.

\subsection{Echoed gradient descent}

We begin by establishing the following properties of gradient descent. In the rest of the theorem and lemma statements in this section $w^*$ is an arbitrary point in $\W$.

\begin{lemma}[Potential-bounded regret for GD]
\label{lem:gdregret}
Let $f$ be a $\beta$-smooth convex function. Then $K$ steps of gradient descent on $f$, with a step size $\eta \leq 1/\beta$, satisfies the potential-bounded regret property with $V(w,s,w^*) := \frac12 \|w-w^*\|^2$:
\[ f(w_{\out}) - f(w^*) \leq \frac{1}{\eta K}\left(\frac{\|w_{\init}-w^*\|^2}{2}-\frac{\|w_{\out}-w^*\|^2}{2} \right).\]
\end{lemma}

%The proof of the above lemma is included in \cref{sec:regret-proofs}. 

\begin{lemma}[Stability of GD]
\label{lem:stab-gd}
For a $\beta$-smooth function $f$, and any $0 \leq \eta \leq 1/\beta$, gradient descent on $f$, run with step size $\eta$ for $K$ steps, is $\epsilon$-uniformly stable with 
$\epsilon = 2\eta K \rho^2/B.$
\end{lemma}

Combining \cref{lem:gdregret,lem:stab-gd}, we conclude the following convergence bound for data-echoed GD:

\begin{theorem}[Data-echoed GD]
\label{thm:degd}
$T$ outer steps of data-echoed gradient descent, with a step size of $\eta = \min\braces{ \frac{1}{\beta}, \frac{\rho}{KD}\sqrt{\frac{B}{T}} }$ and $K$ internal steps, produces a point $w_{out}$ satisfying
\[\E[F(w_{\out})] - F(w^*) \leq \frac{\beta D^2}{2KT} + \frac{2\rho D}{\sqrt{BT}}.\]
\end{theorem}

% \begin{proof}[Proof of Theorem \ref{thm:degd}]
% Substituting the result of \cref{lem:gdregret,lem:stab-gd} in \cref{thm:main}
% gives the following:
% \[
%   \E[F(w_{\out})] - F(w^*) 
%   \leq \frac{2\eta K\rho^2}{B} + \frac{\|w_{\init} - w^*\|^2}{2\eta KT}.
% \] 
% Plugging in the choice of $\eta$ finishes the result.
% \end{proof}

\subsection{Echoed proximal gradient descent}
% We first set up the precise version of proximal gradient descent we use. The algorithm is parametrized by four parameters, a step size $\eta \geq 0$, a prox scalar $\gamma \geq 0$, a pivot and an initial point $s_{\init}, w_{\init} \in \W$. Given a convex function $f$, we now define the prox function $f_{\mathrm{prox}}(w)$ as 
% \[f_{\mathrm{prox}}(w) = f(w) + \frac{\gamma}{2}\|w-s_{\init}\|^2 \]
% We set $w_0 = w_{\init}$ and the $t^{th}$ step of the algorithm is now given by
% \[w_{t+1} = w_t - \eta \nabla f_{\mathrm{prox}}(w_t).\]
% After $T$ steps, the algorithm has the following outputs
% \[w_{\out} = w_T \qquad s_{\out} = \frac{\sum_{t} w_t}{T}\]
For proximal GD, we derive the following bounds on potential-bounded regret and stability:

\begin{lemma}
\label{lem:pgdregret}
Let $f$ be a $\beta$-smooth convex function. Consider the potential function
\[V(w,\{\eta,\gamma,w_{\pivot}\},w^*) = \frac{\|w-w^*\|^2}{2 \eta K} + \frac{ \gamma \|w - w_{\pivot} \|^2}{2}.\] Then
$K$-step proximal gradient descent, with step-size $\eta \leq 1/(\beta +
\gamma)$ has regret bounded by
\[ 
  f(w_{\out}) - f(w^*) \leq V(w_{\out}, s_{\out}, w^*) - V(w_{\init}, s_{\init}, w^*)
  .
\]
\end{lemma}

% The proof is included in \cref{sec:regret-proofs}. 

\begin{lemma}[Stability of prox-GD]
\label{lem:stab-pgd}
For a $\beta$-smooth function $f$, any $\lambda \geq 0$ and any $0 \leq \eta \leq 1/(\beta+\lambda)$, $K$ steps of proximal gradient descent are $\epsilon$-uniformly stable with 
$
  \epsilon 
  = \frac{2\rho^2}{B \gamma} \brk!{ 1 - (1 - \eta \gamma)^K }
  .$
\end{lemma}

The proofs of both lemmas are included in the \cref{sec:stability-proofs}. Combining
\cref{lem:pgdregret,lem:stab-pgd}, we get the following guarantee on the
performance of data-echoed prox-GD (proof included in \cref{app:proofs-theorems}):

\begin{theorem}[Data-echoed prox-GD]
\label{thm:depgd}
$T$ outer steps of echoed gradient descent, with a prox parameter of $\gamma = \frac{\rho}{D}\sqrt{\frac{T}{B}}$, step size $\eta = \frac{1}{\beta+\gamma}$, and $K$ internal steps, produces a point $w_{\out}$ satisfying 
\begin{align*}
  \E[F(w_{\out})] - F(w^*)
  \leq 
  \sqrt{1 + \frac{1}{K}} \cdot \frac{ 2\rho \|w_{\init} - w^*\|}{\sqrt{BT}} + \frac{\beta \|w_{\init} - w^*\|^2}{2KT}
  .
\end{align*}
\end{theorem}

% \begin{proof}[Proof of Theorem \ref{thm:depgd}]
% Substituting the result of \cref{lem:pgdregret,lem:stab-pgd} in \cref{thm:main}
% gives the following
% \begin{align*}
%     \E[F(w_{\out})] - F(w^*)
%     &\leq \frac{2\rho^2}{B \gamma} \brk!{ 1 - (1 - \eta \gamma)^K } + \frac{\gamma \|w_{\init} - w^*\|^2}{2T} + \frac{\|w_{\init} - w^*\|^2}{2\eta KT} 
%     \\
%     &\leq \sqrt{1 + \frac{1}{K}} \cdot \frac{ 2\rho \|w_{\init} - w^*\|}{\sqrt{BT}} + \frac{\beta \|w_{\init} - w^*\|^2}{2KT}
%     .
% \end{align*}
% Plugging in the choice of $\eta$ now finishes the result. 
% \end{proof}

% \paragraph{Lack of dependence on $K$.} 
Note that using this algorithm, the correct choice of step size $\eta$ no
longer depends on the echoing factor $K$. In fact, even if $K$ varies across
the execution of the proximal algorithm, a straightforward extension of our
analysis shows that proximal gradient descent can achieve $\sum_t K_t$ instead
of the $KT$ factor in the denominator of the bias term. This resilience to
indeterminate echoing factors is especially appealing for the motivating setting of asynchronous pipelines.

\subsection{Echoed accelerated gradient descent}

For the case of Nesterov's accelerated gradient descent, we consider a slightly modified version of our data-echoing meta-procedure. This arises from the fact that even the stochastic setting of accelerated gradient descent, algorithms output the final iterate and not the average iterate. The resulting slightly modified procedure is outlined in
\cref{alg:dataecho2}.

\begin{algorithm}[h!]
\caption{Data-echoing meta-algorithm (final iterate)}
\label{alg:dataecho2}
\begin{algorithmic}[1]
\State \textbf{Input: } Optimizer $\A$; initializer $w_{\init} := w_0$; initial state $s_{\init} := s_0$; number of inner steps $K$.
\For{$t = 0, \ldots, T-1$}
  \State Receive a batch of examples $\batch^{(t)} = \{\xi^{(t,i)}\}_{i = 1}^{B}$.
%   \State Define the batch objective
%   \[\fBart{t}(w) \defeq \frac{1}{B} \sum_{i=1}^B f(w, \xi^{(t,i)})\]
  \State Execute $\A$ on $\batch^{(t)}$ starting at $w_t$ for $K$ steps:
  \quad
  $w_{t+1},s_{t+1} \leftarrow \A(w_{t}, s_t, \batch^{(t)}, K_t).$
 \EndFor
\State \textbf{Output: } Final iterate $w_{\out} := w_T$
\end{algorithmic}
\end{algorithm}

We also add a slight extension to our potential-based regret abstraction:
\begin{lemma}[Potential-bounded regret for AGD]
\label{lem:nagdregret}
Let $f$ be a $\beta$-smooth convex function. Running accelerated gradient descent for $K$ steps, with a step size $\eta \leq 1/\beta$, gives the regret bound
\begin{align*}
  &(\lambda_{\out}^2 - \lambda_{\out}) (f(w_{\out}) - f(w)) 
  - (\lambda_{\init}^2 - \lambda_{\init})(f(w_{\init}) - f(w))
  \\
  &\;\;\;%\quad\quad\quad
  \leq 
  \frac{1}{2\eta}(\|w_{\init} + \lambda_{\init} d_{\init} - w\|^2
%   \\
%   &\quad\quad\quad\qquad
   - \|w_{\out} + \lambda_{\out} d_{\out} - w\|^2 )
  .
\end{align*}
\end{lemma}
Further, to bound the stability, we note the following lemma which was essentially proved in \cite{chen2018stability}. Since we believe there is a small typo in the main argument in the original presentation of the proof, we provide an alternate derivation in \cref{sec:stability-proofs}. 

\begin{lemma}[Stability of AGD]
\label{lem:stab-agd}
Suppose that $f$ is a $\beta$-smooth convex quadratic function of $w$ for any $\xi$.
Then, for any $0 \leq \eta \leq 1/\beta$ and initial state $s_{\init}$, $K$
steps of accelerated gradient descent are $\epsilon$-uniformly stable with $\epsilon = O(\eta \rho^2 K^2/B).$
\end{lemma}

Combining \cref{lem:nagdregret,lem:stab-agd} we obtain the following
guarantee for data-echoed AGD:

\begin{theorem}
\label{thm:deagd}
Suppose $f$ is a convex quadratic in $w$, for all $\xi$. Then, $T$ outer steps of echoed AGD, with echoing factor $K$ and step size 
$
  \eta 
  = \Theta\brk{\min\set{ \frac{1}{\beta}, \frac{\rho}{K^2D\sqrt{B/T^{3/2}}} }},
$ produces a point $w_{out}$ satisfying 
\[
  \E[F(w_{\out})] - F(w^*) 
  = O\left(\frac{\beta\|w_0 - w^*\|^2}{K^2T^2} + \frac{\rho \|w_0 - w^*\|}{\sqrt{BT}}\right)
  .
\]
\end{theorem}

\section{Experiments}
We demonstrate numerical experiments on convex machine learning benchmarks. This acts as a validation of our theoretical findings, as well as a way to examine ``beyond worst-case'' phenomena not captured by our minimax convergence guarantees. This can be seen as a combination of the experiments of Figures 4-6 in \cite{choi2019faster}, where we have exchanged the state-of-the-art setting for a more robust one, allowing for a closer dissection of the bias-variance decomposition.

\paragraph{Methodology.} We consider two logistic regression problems as a benchmark, the scaled CoverType dataset from the UCI repository \cite{dua2019uci}, and MNIST \cite{lecun2010mnist}. We record the number of iterations (including as well as excluding the data-echoing iterations) needed for SGD to reach within 1\% of the optimum training loss, as we increase the echoing factor $K$, and thus decrease the \emph{rate} of fresh independent samples usable by SGD. For each choice of $(B,K)$, we tune a constant learning rate by grid search, to minimize this time. All details can be found in \cref{sec:experiment-appendix}.

\paragraph{Results and discussion.} Figures~\ref{fig:covtype-echoing} and \ref{fig:mnist-echoing} show our findings. As batch size $B$ increases, there is a phase transition from a variance-dominated regime (the $O(\rho D / \sqrt{B T})$ term in our analysis is larger) to a bias-dominated regime (the $O(\beta D / KT)$ term is larger). In the former regime, data-echoed SGD saturates on the stale data, and the optimal learning rate scales inversely with $K$, as predicted by the theory. In the latter regime, echoing attains a nearly embarrassingly-parallel speedup, and the optimal learning rate is close to constant. These experiments provide an end-to-end example of how the bias-variance decomposition can help to understand and diagnose the benefits and limitations of data-echoed algorithms.

\paragraph{A note on deep neural nets.} Our theoretical setting was originally motivated by hardware constraints most frequently encountered in the massively parallel training of deep neural networks. Beyond the convex setting, we note that the experimental design problem become significantly more challenging. Some potential confounds include the learning rate choice affecting the generalization gap \cite{jiang2020exploring}, and counterintuitive interactions between learning rate and batch normalization \cite{arora2018theoretical,li2019exponential}. In \cite{choi2019faster}, the authors study the \emph{end-to-end} performance gains of data echoing. Indeed, those experiments need many tweaks (like \emph{example-wise} echoing, data re-augmentation, and individually tuned momentum and learning rate schedules) to obtain their most impressive speedups.

\definecolor{plotyellow}{rgb}{0.5, 0.5, 0}

\begin{figure}
    \centering
    \includegraphics[width=\linewidth]{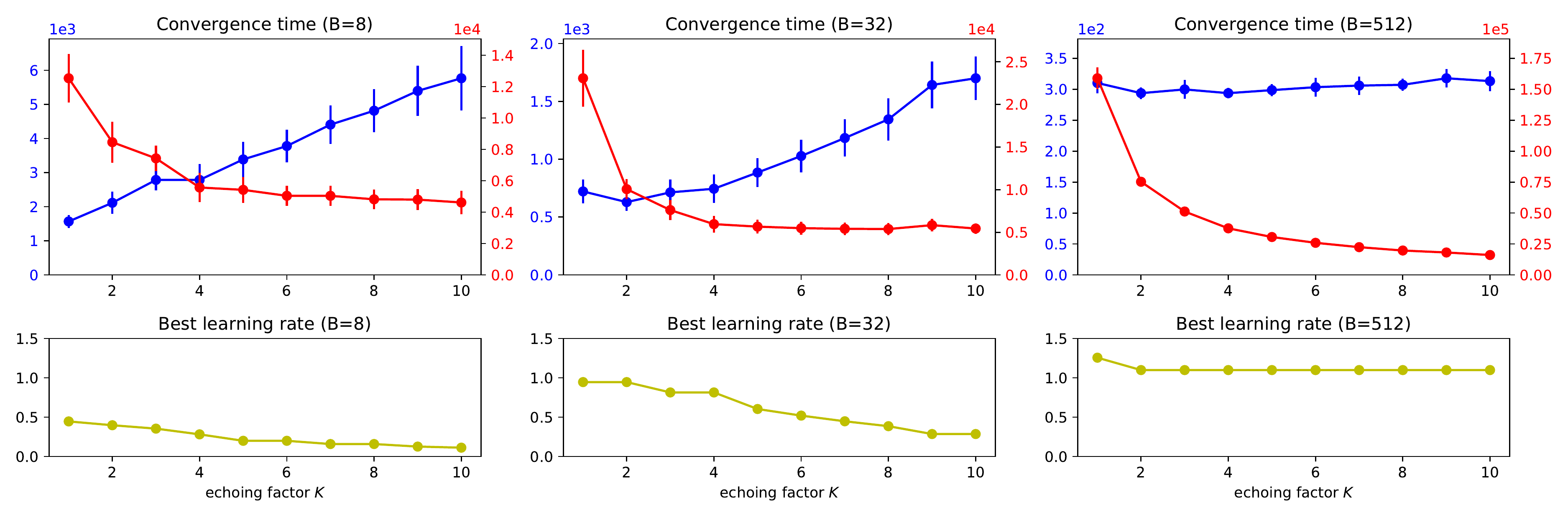}
    \vspace{-5mm}
    \caption{Convergence times as a function of echoing factor $K$, for logistic regression on the CoverType dataset. Learning rates (\textcolor{plotyellow}{yellow}) are tuned for each $(B,K)$ to minimize convergence times. Convergence times are presented in number of SGD steps $KT$ (\textcolor{blue}{blue}), as well as number of independent samples consumed $BT$ (\textcolor{red}{red}). Note that the \textcolor{red}{red} curves reflect wall-clock time for data-echoing when the data loader is $K$ times slower than the optimizer. As batch size $B$ increases, we move from the noise-dominated regime (\textcolor{red}{red curve plateaus}) to the curvature-dominated regime (\textcolor{blue}{blue curve plateaus}).}
    \label{fig:covtype-echoing}
\end{figure}

\begin{figure}
    \centering
    \includegraphics[width=\linewidth]{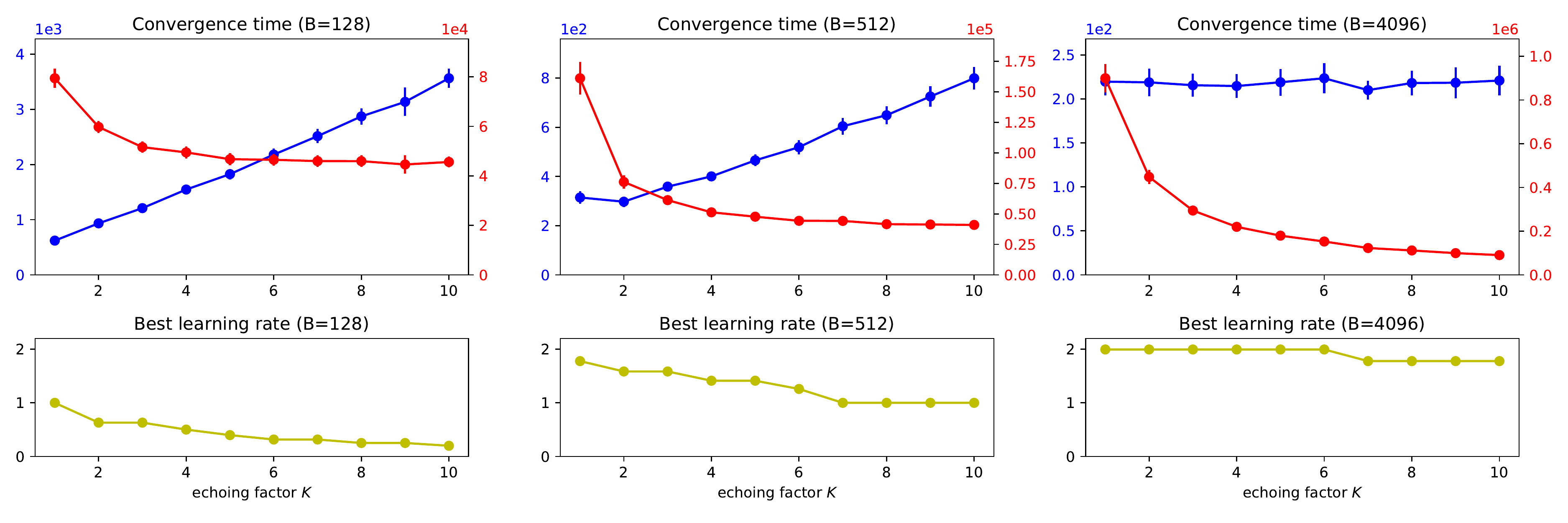}
    \vspace{-5mm}
    \caption{Convergence times, as in Figure~\ref{fig:covtype-echoing}, for logistic regression on the MNIST dataset. Note that the phase transition from noise-dominated to curvature-dominated regimes happens in a batch size range commonly used in deep learning benchmarks with this dataset.}
    \label{fig:mnist-echoing}
\end{figure}

\section{Conclusion}

We have established first theoretical analysis in the nascent field of
optimization algorithms for asynchronous data pipelines, where we have found that gradient descent and well-known variants can be adapted to resist overfitting to stale data. An immediate open problem is to develop a corresponding theory for local convergence and saddle point avoidance in the non-convex setting. This work provides further motivation to show the $O\brk{ \eta \rho^2 K^2 / B }$-uniform stability of AGD for smooth convex functions, which was conjectured in \cite{chen2018stability} with different motives.
More broadly, we hope that the design and analysis of algorithms in optimization for machine learning can derive fruitful inspiration from nascent hardware considerations, like those that motivated this work.

\section*{Acknowledgements}
We are grateful to George Dahl and Yoram Singer for helpful discussions.

\bibliographystyle{abbrvnat}
\bibliography{main}

\newpage

\appendix

\section{Proofs for potential-bounded regret lemmas}
\label{sec:regret-proofs}

In this section we provide the proofs of \cref{lem:gdregret,lem:pgdregret,lem:nagdregret}, which concern potential-bounded regret.

\begin{proof}[Proof of \cref{lem:gdregret}] 
To remind the reader, a step of gradient descent with step-size $\eta$ is given
by 
\[
  w_{j+1} = w_j - \eta \nabla f(w_j)
  ,
\] 
with $w_0 := w_{\init}$ and $w_{\out} := w_K$. Further fix an arbitrary $w^* \in
\W$. Using the definition of $w_{j+1}$ and convexity we get that, 
\begin{align}
f(w_j) - f(w^*) &\leq \nabla f(w_j)(w_j - w^*) \nonumber\\
&\leq \frac{1}{2\eta} \big(\norm{w_j-w^*}^2 - \norm{w_{j+1}-w^*}^2 \big)
	+ \frac{\eta}{2} \norm{\nabla f(w_j)}^2. \label{eq:stdregret1}
\end{align}
Furthermore, using $\beta$-smoothness we get
\begin{align*}
    f(w_{j+1}) - f(w_j)
    &\leq \nabla f(w_j)(w_{j+1} - w_j) + \frac{\beta}{2}\|w_{j+1}-w_j\|^2 \\&\leq - \eta (1-\tfrac{1}{2}\eta\beta) \norm{\nabla f(w_j)}^{2}.
\end{align*}
Therefore for $0 \leq \eta \leq 1/\beta$, we have that 
\begin{align}
    \norm{\nabla f(w_j)}^{2}
	&\leq
	\frac{1}{\eta (1-\tfrac{1}{2}\eta\beta)} \big( f(w_j)-f(w_{j+1}) \big) \nonumber\\
	&\leq
	\frac{2}{\eta} \big( f(w_j)-f(w_{j+1}) \big). \label{eq:descentlemma1}
\end{align} 
Collecting \cref{eq:stdregret1,eq:descentlemma1}, summing and rearranging, we obtain
\begin{align*}
	\sum_{t=0}^{K-1} \big( f(w_{j+1}) - f(w^*) \big)
	\leq
	\frac{1}{2\eta} \big( \norm{w_0-w^*}^2 - \norm{w_K-w^*}^2 \big)
	.
\end{align*}
Finally, observe that \cref{eq:descentlemma1} also implies $f(w_K) \leq f(w_j)$ for all $1 \leq j \leq K$, which now gives the lemma.
\end{proof}

\begin{proof}[Proof of \cref{lem:pgdregret}]
Fix an arbitrary $w^* \in \W$. Using the definition of $w_{j+1}$ and the $\lambda$ strong-convexity of $f_{\mathrm{prox}}$ we get that, 
\begin{align}
f_{\mathrm{prox}}(w_j) - f_{\mathrm{prox}}(w^*) \nonumber &\leq \nabla f_{\mathrm{prox}}(w_j)(w_j - w^*) - \frac{\gamma}{2}\|w_j - w^*\|^2  \nonumber\\
&\leq \frac{1}{2\eta} \big(\norm{w_j-w^*}^2 - \norm{w_{j+1}-w^*}^2 \big)
	+ \frac{\eta}{2} \norm{\nabla f_{\mathrm{prox}}(w_j)}^2  - \frac{\gamma}{2}\|w_j - w^*\|^2. \label{eq:stdregret2}
\end{align}
Furthermore, using the $(\beta + \gamma)$-smoothness of $f_{\mathrm{prox}}$ we get
\begin{align*}
    f_{\mathrm{prox}}(w_{j+1}) - f_{\mathrm{prox}}(w_j)
    &\leq \nabla f_{\mathrm{prox}}(w_j)(w_{j+1} - w_j) + \frac{\beta+\gamma}{2}\|w_{j+1}-w_j\|^2 \\
    &= - \eta (1-\tfrac{1}{2}\eta(\beta+\gamma)) \norm{\nabla f_{\mathrm{prox}}(w_j)}^{2}.
\end{align*}
Therefore, for $0 \leq \eta \leq 1/(\beta+\gamma)$, we have that 
\begin{align}
    \norm{\nabla f_{\mathrm{prox}}(w_j)}^{2}
	&\leq
	\frac{1}{\eta (1-\tfrac{1}{2}\eta(\beta+\gamma))} \big( f_{\mathrm{prox}}(w_j)-f_{\mathrm{prox}}(w_{j+1}) \big) \nonumber\\
	&\leq
	\frac{2}{\eta} \big( f_{\mathrm{prox}}(w_j)-f_{\mathrm{prox}}(w_{j+1}) \big). \label{eq:descentlemma2}
\end{align} 
Collecting \cref{eq:stdregret2,eq:descentlemma2}, summing and rearranging, we
obtain
\begin{align*}
	\frac{1}{K} \sum_{t=0}^{K-1} \big( f_{\mathrm{prox}}(w_{j+1}) - f_{\mathrm{prox}}(w^*) \big)
	&\leq
	\frac{1}{2\eta K} \big( \norm{w_0-w^*}^2 - \norm{w_K-w^*}^2 \big) 
	- \sum_{j=0}^{K-1} \frac{\gamma}{2K}\|w_j - w^*\|^2
	\\
	&\leq
	\frac{1}{2\eta K} \big( \norm{w_0-w^*}^2 - \norm{w_K-w^*}^2 \big) 
	- \frac{\gamma}{2}\left\|\frac{\sum_j w_j}{K} - w^*\right\|^2
	\\
	&\leq
	\frac{1}{2\eta K} \big( \norm{w_0-w^*}^2 - \norm{w_K-w^*}^2 \big) 
	- \frac{\gamma}{2}\left\|s_{\out} - w^*\right\|^2.
\end{align*}
Finally, observe that \cref{eq:descentlemma2} also implies $f_{\mathrm{prox}}(w_K) \leq f_{\mathrm{prox}}(w_j)$ for all $j$.  Therefore we have that
\begin{align*}
	f(w_{\out}) - f(w^*) - \frac{\gamma}{2}\|w^* - s_{\init}\|^2 &\leq
	\frac{1}{K} \sum_{t=0}^{K-1} \big( f_{\mathrm{prox}}(w_{j+1}) - f_{\mathrm{prox}}(w^*) \big)
	\\
	&\leq
	\frac{1}{2\eta K} \big( \norm{w_0-w^*}^2 - \norm{w_K-w^*}^2 \big) 
	- \frac{\gamma}{2}\left\|s_{\out} - w^*\right\|^2,
\end{align*}
This concludes the lemma.
\end{proof}

\begin{proof}[Proof of \cref{lem:nagdregret}]
Let $x_{j+1} \defeq w_j + d_j$. Fix an arbitrary $w^* \in \W$ and define $h(w) \defeq f(w)-f(w^*)$. We will now collect a host of inequalities that will be useful. 
First by smoothness and the choice of $\eta$ we have
\begin{align*}
    f(w_{j+1}) - f(x_{j+1}) \leq -\frac{\eta}{2}\|\nabla f(x_{j+1})\|^2 
    .
\end{align*}
Further, by convexity we have
\begin{align*}
    f(x_{j+1}) - f(w_j) &\leq \nabla f(x_{j+1})^{\top} d_j ;
% \end{align*}
  \\
% \begin{align*}
   f(x_{j+1}) - f(w^*) &\leq \nabla f(x_{j+1})^{\top}(w_j + d_j - w^*) 
   .
\end{align*}
Adding the above we get 
\begin{align}
    h(w_{j+1}) - h(w_j) 
    &\leq -\frac{\eta}{2}\|\nabla f(x_{j+1})\|^2 + \nabla f(x_{j+1})^{\top}d_j ;
    \label{eqn:ineq1}
    \\
% \end{align}
% \begin{align}
    h(w_{j+1}) 
    &\leq -\frac{\eta}{2}\|\nabla f(x_{j+1})\|^2 + \nabla f(x_{j+1})^{\top}(w_j + d_j - w^*) .
    \label{eqn:ineq2}
\end{align}
Furthermore, note that 
\begin{align} \label{eqn:ineq3}
  \lambda_j \|\eta \nabla f(x_{j+1})\|^2 &+ 2\eta \nabla f(x_{j+1})^{\top}(w_j + \lambda d_j - w) \nonumber \\
  &= \frac{1}{\lambda_j}\left( \|w_j + \lambda_j d_j - w^* + \lambda_j \eta \nabla f(w_{j+1})\|^2 - \|w_j + \lambda_j d_j - w^*\|^2\right) \nonumber \\
  &= \frac{1}{\lambda_j}\left( \|w_{j+1} + \lambda_{j+1} d_{j+1} - w^*\|^2 - \|w_j + \lambda_j d_j - w^* \|^2\right)
  .
\end{align}
Adding $(\lambda_j - 1)$ times \cref{eqn:ineq1}, 1 times \cref{eqn:ineq2} and $(-1/2\eta)$ times \cref{eqn:ineq3} gives us
\[
  \lambda_j^2 h(w_{j+1}) - (\lambda_j^2 - \lambda_j)h(w_j) \leq \frac{1}{2\eta}
(u_j - u_{j+1})
,
\] 
where $u_j = \|w_j + \lambda_j d_j - w^*\|^2$. Summing this over time we get
\begin{align*}
  (\lambda_{K}^2 - \lambda_K) h(w_{K}) - (\lambda_0^2 - \lambda_0)h(w_0)
  &= \lambda_{K-1}^2 h(w_{K}) - (\lambda_0^2 - \lambda_0)h(w_0) \\
  &\leq \frac{1}{2\eta}(\|w_0 + \lambda_0 d_0 - w^*\|^2 - \|w_K + \lambda_K d_K - w^*\|^2)  
\end{align*}
which finishes the proof. 

\end{proof}

'

\section{Stability proofs}
\label{sec:stability-proofs}
In this section, we prove the bounds on the stability of the respective algorithms (\cref{lem:stab-gd,,lem:stab-pgd,,lem:stab-agd}). Our general recipe for showing stability of various algorithms would be to show that the points visited by the iterative algorithms themselves do not differ by much. %To this end, given two samples $\batch$ and $\batch'$ differing in exactly one example, let the points generated by algorithm $\A$ when running on samples $\batch$ and $\batch'$ be denoted by 
% \[\{w_{t}^{\batch}\} \leftarrow \A(\fBart{t}(\batch)), \{w_{t}^{\batch'}\} \leftarrow \A(\fBart{t}(\batch'))\]
To this end, note that since $f$ is Lipschitz, we have that
\begin{equation}
\label{eqn:lipbound}
    \sup_{\xi \in \Xi} |f(\A(\batch), \xi) - f(\A(\batch'), \xi)| \leq \rho \|\A(\batch) - \A(\batch')\|
    .
\end{equation}

Thus, it is sufficient to show to bound $\A(\batch) - \A(\batch')$, which is what we do next. 

\begin{proof}[Proof of \cref{lem:stab-gd}]
For simplicity of presentation we assume that the Hessian of $f$ is a continuous function. The more general case can be derived by following the arguments in \cite{hardt2015train}. 

Let $w_j^{\batch}$ and $w_j^{\batch'}$ denote the points generated by gradient descent on $\batch$ and $\batch'$ respectively. Further define 
\[\Delta w_j := w_j^{\batch'} - w_j^{\batch};\]
\[\nabla \fBar_{\batch}(\Delta w_j) := \nabla \fBar(w_j^{\batch'}) - \nabla \fBar(w_j^{\batch}).\]
Therefore via the mean value theorem, we can write this as 
\[\nabla \fBar_{\batch}(\Delta w_j) = H_j(\Delta w_j)\]
for some $H_j$ along the line segment from $w_j^{\batch}$ to $w_j^{\batch'}$. It is now easy to see
\[\Delta w_{j+1} = (I - \eta H_j)\Delta w_{j} + \frac{\eta}{B}(\nabla f(w_j^{\batch'}, \xi') - \nabla f(w_j^{\batch'}, \xi)).\]
Noting that $0 \preceq I - \eta H_j \preceq I$(by the choice of $\eta$) and that $\|\nabla f\| \leq \rho$, we have that
\[\|\Delta w_{j+1}\| \leq \|\Delta w_{j}\| + \frac{2 \eta \rho}{B}.\]
The proof is now finished by using \cref{eqn:lipbound}.
\end{proof}

\begin{proof}[Proof of \cref{lem:stab-pgd}]
The proof follows the exact same structure as in the proof of \cref{lem:stab-gd}, except at the end where since the prox function is $\lambda$ strongly convex we get that
\[0 \preceq I - \eta H_j \preceq (1 - \eta\gamma)I.\]
Replacing this gives us 
\[\|\Delta w_{j+1}\| \leq (1 - \eta\gamma)\|\Delta w_j\| + \frac{2 \eta \rho}{B}.\]
Unrolling the above over $0 \leq j \leq K-1$ and using \cref{eqn:lipbound} gives the result. 
\end{proof}

\begin{proof}[Proof of \cref{lem:stab-agd}]
As mentioned before the proof follows exactly along the lines of the proof of Theorem 11 in \cite{chen2018stability}. It can easily be seen from the original proof that the presence of an initial momentum term $d_0$ (which is assumed to be 0 in the original proof) makes no difference to the arguments. Furthermore starting the $\lambda$ sequence from $\lambda_j$ also does not make any difference to the proof, as it only requires $-1 \leq \gamma_j \leq 0$ which our sequence also continues to satisfy irrespective of the choice of $\lambda_0$. 

We believe that there is a small typo in the main argument of the original proof in Lemma~20 in~\cite{chen2018stability}. We fix the slight indexing error of the argument. In particular, the proof boils down to showing the following lemma. 

\begin{lemma}[Lemma 20, \citealp{chen2018stability}]
\label{lem:chebyshev}
Suppose 
$$H_i = \begin{bmatrix}
(1 - \gamma_i)h & \gamma_i h \\
1 & 0
\end{bmatrix}$$ where $h \in [0,1]$ and $\gamma_i \in (-1,1).$ 
Then, for all $t \in \mathbb{N}$,
\[\left\|\prod_{i=1}^j H_i\right\|_2 \leq 2(j+1).\]
\end{lemma}

The proof of the lemma proceeds by analyzing the cases when $\y_i \in \{-1,1\}$. The only case where we differ from the presented proof is when $\gamma_i = -1$. In this case, we have 
$$
H_i = H := \begin{bmatrix}
2h & -h \\
1 & 0
\end{bmatrix},
$$
and we need to bound the operator norm of the powers of this matrix for all $h \in [0,1]$.

\begin{lemma}
\label{lem:cheb-powers}
For any $n \geq 1$, we have
\[ H^{2n} = h^n \begin{bmatrix} U_{2n}(\sqrt{h}) & -\sqrt{h} \cdot U_{2n-1}(\sqrt{h}) \\ \frac{U_{2n-1}(\sqrt{h})}{\sqrt{h}} & -U_{2n-2}(\sqrt{h}) \end{bmatrix}, \]
and
\[ H^{2n+1} = h^n \begin{bmatrix} \sqrt{h} U_{2n+1}(\sqrt{h}) & -h \cdot U_{2n}(\sqrt{h}) \\ U_{2n}(\sqrt{h}) & -\sqrt{h} U_{2n-1}(\sqrt{h}) \end{bmatrix}, \]
where $U_n(\cdot)$ is the $n$-th Chebyshev polynomial of the second kind.
\end{lemma}
\begin{proof}
We begin by proving the identity for even powers $2n$, by induction. The base case $n=1$ holds by manual computation, noting the following facts:
\[ H^2 = h \begin{bmatrix}
4h-1 & -2h \\
2 & -1
\end{bmatrix},
\]
\[ U_0(x) = 1, \quad U_1(x) = 2x, \quad U_2(x) = 4x^2-1.
\]
Next, we prove the inductive step, showing that the identity for $2n$ implies the same for $2n+2$.
Below, we substitute $r := \sqrt{h}$ for clarity:
\begin{gather*}
    H^{2n+2} = h \cdot \begin{bmatrix} 4h-1 & -2h \\ 2 & -1 \end{bmatrix} H_i^{2n} \\
    = h^{n+1} \begin{bmatrix} 4r^2-1 & -2r^2 \\ 2 & -1 \end{bmatrix} \begin{bmatrix} U_{2n}(r) & -r U_{2n-1}(r) \\ \frac{U_{2n-1}(r)}{r} & -U_{2n-2}(r) \end{bmatrix},
\end{gather*}
Computing each entry of the matrix product, and applying the recurrence
$U_{n+1}(r) = 2rU_n(r) - U_{n-1}(r)$:
\begin{align*}
    \bra{ H^{2n+2} }_{11} &= h^{n+1} \pa{ (4r^2-1) U_{2n}(r) - 2r U_{2n-1}(r) } \\
    &= h^{n+1} \pa{ 2r U_{2n+1}(r) - U_{2n}(r) } = h^{n+1} U_{2n+2}(r), \\
    \bra{ H^{2n+2} }_{12} &= -h^{n+1} \pa{ r (4r^2-1) U_{2n-1}(r) - 2r^2 U_{2n-2}(r) } \\
    &= -h^{n+1} r \bra{ H^{2n-1} }_{11} =  h^{n+1} r U_{2n+1}(r),  \\
    \bra{ H^{2n+2} }_{21} &= h^{n+1} \cdot \frac{ 2r U_{2n}(r) - U_{2n-1}(r)}{r} = h^{n+1} \frac{ U_{2n+1}(r) }{r}, \\
    \bra{ H^{2n+2} }_{22} &= -h^{n+1} \pa{ 2r U_{2n-1}(r) - U_{2n-2}(r) } = -h^{n+1} U_{2n}(r).
\end{align*}
This concludes the claimed identity for the even case. Finally, we show that the $2n+1$ case follows from the $2n$ case:
\begin{align*}
    H^{2n+1}
    = h^n \begin{bmatrix} 2r^2 & -r^2 \\ 1 & 0 \end{bmatrix} \begin{bmatrix} U_{2n}(r) & -r U_{2n-1}(r) \\ \frac{U_{2n-1}(r)}{r} & -U_{2n-2}(r) \end{bmatrix},
\end{align*}
so that
\begin{align*}
    \bra{ H^{2n+1} }_{11} &= h^n \pa{2r^2 U_{2n}(r) - r U_{2n-1}(r)}
    = h^n \cdot r U_{2n+1}(r), \\
    \bra{ H^{2n+1} }_{12} &= - h^n \pa{2r^3 U_{2n-1}(r) - r^2 U_{2n-2}(r)}
    = - h^n \cdot h U_{2n}(r), \\
    \bra{ H^{2n+1} }_{21} &= U_{2n}(r), \\
    \bra{ H^{2n+1} }_{22} &= -rU_{2n-1}(r).
\end{align*}
This completes the proof of the odd case, hence Lemma~\ref{lem:cheb-powers}.
\end{proof}

To finish the proof of Lemma~\ref{lem:chebyshev}, we use the classical fact that $|U_j(r)| \leq j+1$ for all $|r| \leq 1$,
and note that each entry of $H^j$ is the value of some $U$, times a scalar between $-1$ and $1$; the $1/r$ factor in $[H^{2n}]_{21}$ gets absorbed because $h/r = r \leq 1$. This shows that for all $j \geq 2$, each entry of the $2 \times 2$ matrix $H^j$ has absolute value bounded by $|U_{j+1}(r)| \leq j+1$; the same can be verified manually for $j=1$. We conclude \cref{lem:chebyshev} by bounding $\norm{H^j}_2 \leq \norm{H^j}_1 \leq 2(j+1)$.
\end{proof}

\section{Proofs of the main theorems}
\label{app:proofs-theorems}
In this section we use the potential-bounded regret and stability lemmas to complete the proofs of Theorems~ \ref{thm:degd}, \ref{thm:depgd}, and \ref{thm:deagd}.

\begin{proof}[Proof of Theorem \ref{thm:degd}]
Substituting the result of \cref{lem:gdregret,lem:stab-gd} in \cref{thm:main}
gives the following:
\[
  \E[F(w_{\out})] - F(w^*) 
  \leq \frac{2\eta K\rho^2}{B} + \frac{\|w_{\init} - w^*\|^2}{2\eta KT}.
\] 
Plugging in the choice of $\eta$ concludes the result.
\end{proof}

\begin{proof}[Proof of Theorem \ref{thm:depgd}]
Substituting the result of \cref{lem:pgdregret,lem:stab-pgd} in \cref{thm:main}
gives the following
\begin{align*}
    \E[F(w_{\out})] - F(w^*)
    &\leq \frac{2\rho^2}{B \gamma} \brk!{ 1 - (1 - \eta \gamma)^K } + \frac{\gamma \|w_{\init} - w^*\|^2}{2T} + \frac{\|w_{\init} - w^*\|^2}{2\eta KT} 
    \\
    &\leq \sqrt{1 + \frac{1}{K}} \cdot \frac{ 2\rho \|w_{\init} - w^*\|}{\sqrt{BT}} + \frac{\beta \|w_{\init} - w^*\|^2}{2KT}
    .
\end{align*}
Plugging in the choice of $\eta$ now concludes the result. 
\end{proof}

\begin{proof}[Proof of Theorem~\ref{thm:deagd}]
We will use the notation $\lambda_{t}, d_t$ to denote the $\lambda_{\out}, d_{\out}$ returned by $\A$ at iteration $t$ of Algorithm~\ref{alg:dataecho}. From \cref{lem:nagdregret}
we get that
% \begin{multline*} \fBar_{\batch^{(t)}}(w_{t+1}) - \fBar_{\batch^{(t)}}(w^*)
%   \\\leq V_{\A}(w_{t}, s_{t}, w^*) - V_{\A}(w_{t+1}, s_{t+1}, w^*)  
% \end{multline*}
\begin{align*}
  (\lambda_{t}^2 - \lambda_{t}) (\fBar_{\batch^{(t)}}(w_{t+1}) 
    &- \fBar_{\batch^{(t)}}(w^*)) - (\lambda_{t-1}^2 - \lambda_{t-1})(\fBar_{\batch^{(t)}}(w_{t-1}) - \fBar_{\batch^{(t)}}(w^*))\\
  &\leq \frac{1}{2\eta}(\|w_{t} + \lambda_{t} d_{t} - w^*\|^2 - \|w_{t+1} + \lambda_{t+1} d_{t+1} - w^*\|^2).
\end{align*}

Let $\E_t[\cdot]$ be the expectation conditioned with respect to the randomness
up time $t$ (inclusive). We now get from the uniform stability of $\A$ that
\begin{align*}
  \E[F(w_{t+1}) - F(w^*)]
  &= \E_{t-1}[\E_{\batch^{(t)}}[F(w_{t+1}) -F(w^*)]] \\
  &= \E_{t-1}[\E_{\batch^{(t)}}[\fBar_{\batch^{(t)}}(w_{t+1}) - \fBar_{\batch^{(t)}}(w^*)] ] + O\left(\frac{\eta \rho^2 K^2}{B}\right).
\end{align*}
Using the above inequalities, appropriately scaling and summing over $t$ and noting that $\lambda_0 = 1$ we get
\begin{align*}
  \lambda_{T-1}^2 \pa{ \E[F(w_T)] - F(w^*) }
  = \frac{\|w_0 - w^*\|^2}{2\eta} + O\left(\frac{\eta \rho^2 K^2 \sum_t \lambda_t^2}{B}\right).
\end{align*}
Using standard bounds on $\lambda_t$, we get that $\lambda_t = \Theta(tK)$. Substituting this in the above equation gives
\begin{align*}
  \E[F(w_T)] - F(w^*) 
  = O\left(\frac{\|w_0 - w^*\|^2}{2\eta K^2 T^2} + \frac{\eta \rho^2 K^2 T}{B}\right).
\end{align*}
Now, using the value of $\eta$ prescribed in the theorem, we conclude the result. 
\end{proof}

\section{Experiment Details}
\label{sec:experiment-appendix}

\subsection{Datasets}
\paragraph{CoverType.} We used the scaled binary classification version of this dataset, as provided as a benchmark alongside \texttt{libsvm}. This dataset contains 581012 labeled examples, with feature dimension 54; thus, the logistic regression model has 110 parameters (including biases). Since this work is not concerned with generalization on holdout validation data, and this dataset does not come with a canonical train/test split, we trained on all of the examples. However, we note that logistic regression underfits to this dataset; the generalization gap was negligible when we tried random 90\%-10\% splits, and did not affect the trends seen in Figure~\ref{fig:covtype-echoing}.

\paragraph{MNIST.} We used the training set of MNIST, which contains 60000 examples. The feature dimension is 764, and there are 10 classes, for a total of 7650 parameters (including biases). The pixels were normalized to the range $[0,1]$. Again, the generalization gap is negligible in this setting; the results do not change (and the specific convergence times change only slightly) upon computing the convergence criterion using the canonical holdout validation set of 10000 examples.

In all experiments, batches were sampled with replacement (rather than the usual per-epoch shuffling convention), to remove artifacts arising from non-independence.

\subsection{Measuring convergence time}
Thresholds for convergence were chosen to lie within 1\% of the globally optimal training loss. We used 0.54 for CoverType and 0.3 for MNIST. We remark that although these choices are arbitrary, the trends exhibited in our experiments were not sensitive to the precise choice of threshold (although the convergence times can be dramatically different). To reduce variance, we record convergence when the mean of the past 10 losses lies below the threshold. Again, the trends in our experiments were not sensitive to this choice of aggregation. The means and standard deviations of convergence times in Figures~\ref{fig:covtype-echoing} and \ref{fig:mnist-echoing} were computed over 20 runs.

\subsection{Hyperparameters}
Learning rates were selected by grid search over an exponential grid (i.e. \texttt{numpy.logspace}) between $0.01$ and $10$, where consecutive candidates were $10^{1/20}$ apart.

The logistic regression models were trained with bias parameters; all parameters were initialized at zero.

\subsection{Computing infrastructure}
To enable rapid evaluation of training losses on these $\sim\!100$MB datasets, all optimization experiments were implemented in PyTorch on an NVIDIA V100 GPU machine. Each individual run took less than 1 minute.

\end{document}